\documentclass[11pt]{article}
\usepackage{times}
\usepackage{helvet}
\usepackage{courier}
\usepackage{amssymb}
\usepackage{amsmath}
\usepackage{graphicx}
\usepackage{chicagor}

\newtheorem{theorem}{Theorem}[section]
\newtheorem{proposition}[theorem]{Proposition}

\newtheorem{exam}[theorem]{Example}
\newcommand{\bbox}{\vrule height7pt width4pt depth1pt}
\newenvironment{example}{\begin{exam} \rm }{\hfill $\bbox$ \end{exam}}
\newtheorem{rema}[theorem]{Remark}

\newtheorem{obs}[theorem]{Observation}

\newtheorem{defin}[theorem]{Definition}
\newenvironment{definition}{\begin{defin} \rm }{\hfill $\bbox$
\end{defin}}
\newenvironment{proof}{\noindent {\bf Proof.}}{\hfill $\bbox$\\}

\newcommand{\Oo}{\Omega}
\newcommand{\oo}{\omega}
\newcommand{\ffi}{\varphi}

\newcommand{\vDasho}{\vDash^{\ou}}
\newcommand{\ou}{\mathit{out}}
\newcommand{\vDashi}{\vDash^{\inner}}
\newcommand{\inner}{\mathit{in}}
\newcommand{\vDashos}{\vDash^{\ous}}
\newcommand{\ous}{\mathit{out,ai}}
\newcommand{\vDashis}{\vDash^{\inners}}
\newcommand{\inners}{\mathit{in,ai}}

\newcommand{\M}{\mathcal{M}}
\newcommand{\Mis}{\mathcal{M}^{\mathit{ai}}}
\newcommand{\N}{\mathcal{N}}
\newcommand{\Mcpa}{\mathcal{M}^{\mathit{cpa}}}
\newcommand{\Miscpa}{\mathcal{M}^{\mathit{cpa,ai}}}

\newcommand{\F}{\mathcal{F}}
\newcommand{\B}{{\mathcal B}}

\newcommand{\pr}{{\mathit{pr}}}
\newcommand{\commentout}[1]{}

\renewcommand{\iff}{\Leftrightarrow}
\newcommand{\true}{\mathit{true}}
\newcommand{\rec}{\mathit{rec}}
\newcommand{\Supp}{\mathit{Supp}}

\title{Ambiguous Language and Differences in Beliefs\thanks{This paper
wil appear in the \emph{Proceedings of the Thirteenth International
Conference on  Principles of Knowledge Representation and Reasoning (KR
2012)}.}}

\author{{Joseph Y. Halpern}\\
Computer Science Dept.\\
Cornell University\\
Ithaca, NY\\
E-mail: halpern@cs.cornell.edu\\
\and
Willemien Kets\\
Kellogg School of Management\\
Northwestern University\\
Evanston, IL \\
E-mail: w-kets@kellogg.northwestern.edu
}

\begin{document}
\maketitle

\begin{abstract}
Standard models of multi-agent modal logic do not capture the fact that
information is often \emph{ambiguous}, and may be interpreted in different
ways by different agents. We propose a framework that can model this,
and consider different semantics that capture different assumptions
about the agents' beliefs regarding whether or not there is ambiguity.
We consider the impact of ambiguity on a seminal result in
economics: Aumann's result saying that agents with a common prior cannot
agree to disagree.  This result is known not to hold if agents do not have a
common prior;  we show that it also does not hold in the presence of
ambiguity.  We then consider the tradeoff between assuming a \emph{common
interpretation} (i.e., no ambiguity) and a common prior (i.e., shared
initial beliefs).
\end{abstract}

\section{Introduction} \label{sec:intro}

In the study of multi-agent modal logics, it is always implicitly
assumed that all agents interpret all
formulas the same way.  While they may have different beliefs regarding
whether a formula $\ffi$ is true, they agree on what $\ffi$ means. Formally, this is captured by the fact that the truth of $\ffi$ does not depend on the agent.

Of course, in the real world, there is \emph{ambiguity}; different
agents may interpret the same utterance in different ways.
For example, consider a public announcement $p$. Each player $i$ may
interpret $p$ as corresponding to some event $E_i$, where $E_i$ may be
different from $E_j$ if $i \ne j$.  This seems natural: even if
people have a common background, they may still disagree on how to
interpret certain phenomena or new
information. Someone may interpret a smile as just a sign of
friendliness; someone else may interpret it as a ``false'' smile,
concealing contempt; yet another person may interpret it as a sign of
sexual interest.

To model this formally, we can use a straightforward approach already
used in \cite{Hal43,GroveH2}: formulas are interpreted
relative to a player.  But once we allow such ambiguity, further
subtleties arise.  Returning to the announcement $p$, not only can it be
interpreted differently by different players, it may not even occur
to the players that
others may interpret the announcement in a different way.
Thus, for example, $i$ may believe that $E_i$ is
common knowledge.
The assumption that each player believes that her interpretation is how
everyone interprets the announcement is but one assumption we can make
about ambiguity. It is also possible that player $i$ may be aware that
there is more than one interpretation of $p$, but believes that player
$j$ is aware of only one interpretation. For example, think of a
politician making an ambiguous statement which he realizes that
different constituencies will interpret differently, but will not
realize that there are other possible interpretations.
In this paper, we investigate a number of different semantics of
ambiguity that correspond to some standard assumptions that people make
with regard to ambiguous statements, and investigate their
relationship.

Our interest in
ambiguity is motivated by a seminal result in game theory: Aumann's
\citeyear{Aumann_1976} theorem showing that players cannot ``agree to
disagree''.  More precisely, this theorem says that agents with a common
prior on a state space cannot have common knowledge that they have
different posteriors.%
\footnote{We explain this result in more detail in Section~\ref{sec:examples}.}
This result has been viewed as paradoxical in the economics literature.
Trade in a stock market seems to require common knowledge of
disagreement (about the value of the stock being traded), yet we clearly
observe a great deal of trading.

One well known explanation for the disagreement is that we do not in
fact have common priors: agents start out with different beliefs.  We
provide a different explanation here, in terms of ambiguity.
 It is easy to show that we can agree to disagree when
there is ambiguity, even if there is a common prior.
We then show that these two explanations of the possibility of agreeing
to disagree are closely related, but not identical. We can convert an
explanation in
terms of ambiguity to an explanation in terms of lack of common
priors.\footnote{More precisely, we can convert a model with ambiguity
and a common prior to an equivalent model---equivalent in the sense that
the same formulas are true---where there is no ambiguity but no common
prior.}
Importantly, however, the converse does not hold; there are models in
which players have a common interpretation that cannot in general be
converted into an equivalent model with ambiguity and a common prior. In
other words, using heterogeneous priors may be too permissive if we are
interested in modeling a situation where differences in beliefs are due
to differences in interpretation.

Although our work is motivated by applications in economics,
ambiguity has long been a concern in linguistics
and natural language processing.  For example, there has been
a great deal of work on word-sense
disambiguation (i.e., trying to decide from context which of the
multiple meanings of a word are intended);
see Hirst \citeyear{Hirst_1988} for a seminal contribution, and Navigli
\citeyear{Navigli09} for a recent survey.
However, there does not seem to be much work on incorporating ambiguity
into a logic.
Apart from the literature on the logic of context and on
underspecification (see Van Deemter and Peters
\citeyear{VanDeemterPeters_1996}),
the only paper that we are aware of that does this is
one by Monz \citeyear{Monz99}.  Monz allows for statements that have
multiple interpretations, just as we do.
But rather than
incorporating the ambiguity directly into the logic, he considers updates
by ambiguous statements.  There are also connections between ambiguity
and vagueness.  Although the two notions are different---a term is
\emph{vague} if it is not clear what its meaning is, and
is \emph{ambiguous} if it can have multiple meanings, Halpern
\citeyear{Hal43} also used agent-dependent interpretations in his
model of vagueness, although the issues that arose were quite different
from those that concern us here.

The rest of this paper is organized as follows.
 Section \ref{sec:model} introduces the logic that we consider. Section
\ref{sec:examples} investigates the implications of the common-prior
assumption when there is ambiguity. Section \ref{sec:equivalence}
studies the tradeoff between heterogeneous priors and ambiguity, and
Section \ref{sec:concl} concludes.

\section{Syntax and Semantics} \label{sec:model}

\subsection{Syntax}

We want a logic where players use a fixed common language,
but each player may interpret formulas in the language differently.
We also want to allow the players to be able to reason about
(probabilistic) beliefs.

The syntax of the logic is straightforward
(and is, indeed, essentially the syntax already used in papers going
back to Fagin and Halpern \citeyear{FH3}).
There is a finite, nonempty set $N = \{1, \ldots, n\}$ of players, and
a countable, nonempty set $\Phi$ of primitive propositions. Let
$\mathcal{L}_n^C(\Phi)$ be the set of formulas
that can be constructed starting from $\Phi$, and closing off under
conjunction, negation, the modal operators
$\{CB_G\}_{G \subseteq N, G \neq \emptyset}$, and the formation
of probability formulas.
(We omit the $\Phi$ if it is irrelevant or clear from context.)
  Probability formulas are
constructed as follows. If
$\ffi_1, \ldots, \ffi_k$ are formulas, and $a_1, \ldots, a_k, b
\in \mathbb{Q}$, then for $i \in N$,
\[
a_1 \pr_i(\ffi_1) + \ldots + a_k \pr_i(\ffi_k) \geq b
\]
is a probability formula,
where $\pr_i(\ffi)$ denotes the probability that player $i$
assigns to a formula $\ffi$.
Note that this syntax
allows for nested probability formulas. We use the abbreviation
$B_i \ffi$ for $\pr_i (\ffi) = 1$, $EB_G^1 \ffi$ for $\wedge_{i \in G} B_i
\ffi$, and
$EB^{m+1}_G\ffi$ for $EB^m_G EB^1_G\ffi$ for $m = 1,2 \ldots$.
Finally, we take $\true$ to be the abbreviation for a fixed tautology
such as $p \vee \neg p$.

\subsection{Epistemic probability structures}

There are standard approaches for interpreting this language \cite{FH3},
but they all assume that there is no ambiguity, that
is, that all players interpret the primitive propositions the same way.
To allow for different interpretations, we use an approach used earlier
\cite{Hal43,GroveH2}: formulas are interpreted
relative to a player.

An \emph{(epistemic probability) structure} (\emph{over $\Phi$})
has the form
\[
M = (\Oo, (\Pi_j)_{j \in N}, ({\cal P}_j)_{j \in N}, (\pi_j)_{j \in N}),
\]
where $\Oo$ is the state space, and for each $i \in N$, $\Pi_i$
is a partition of $\Oo$, ${\cal P}_i$ is a function that
assigns to each $\oo \in \Oo$ a probability space ${\cal
P}_i(\oo) = (\Oo_{i,\oo}, \F_{i,\oo}, \mu_{i,\oo})$,
and $\pi_i$ is an interpretation that associates with each
state a truth assignment to the primitive propositions in $\Phi$. That
is, $\pi_i(\oo)(p) \in \{\mathbf{true}, \mathbf{false}\}$ for
all $\oo$ and each primitive proposition $p$.
Intuitively, $\pi_i$ describes player $i$'s interpretation of the
primitive propositions. Standard models use only a single
interpretation $\pi$; this is equivalent in our framework to assuming
that $\pi_1 = \cdots = \pi_n$. We call a structure where $\pi_1 = \cdots = \pi_n$ a
\emph{common-interpretation structure}.
Denote by $[[p]]_i$ the set of states where $i$ assigns the value
$\mathbf{true}$ to $p$.
The partitions $\Pi_i$ are called \emph{information partitions}.
While it is more standard in the philosophy and computer science
literature to use models where there is a binary relation ${\cal K}_i$
on $\Oo$ for each agent $i$ that describes $i$'s accessibility relation
on states,
we follow the common approach in economics
of working with information partitions here, as that makes it
particularly easy to define a player's probabilistic beliefs. Assuming
information partitions corresponds to the case that ${\cal K}_i$ is an
equivalence relation (and thus defines a partition).  The intuition is
that a cell in the partition $\Pi_i$ is defined by some information that
$i$ received, such as signals or observations of the world.
Intuitively, agent $i$
receives the same information at each state in a cell of $\Pi_i$.
Let $\Pi_i(\oo)$ denote the cell of the partition $\Pi_i$ containing $\oo$.
Finally, the probability space ${\cal
P}_i(\oo) = (\Oo_{i,\oo}, \F_{i,\oo}, \mu_{i,\oo})$ describes the
beliefs of player $i$ at state $\oo$, with $\mu_{i,\oo}$ a probability
measure defined on the subspace $\Oo_{i,\oo}$ of the state space $\Oo$.
The $\sigma$-algebra $\F_{i, \oo}$ consists of the subsets of
$\Oo_{i,\oo}$ to which $\mu_{i,\oo}$ can assign a probability.  (If
$\Oo_{i,\oo}$ is finite, we typically take $\F_{i,\oo} =
2^{\Oo_{i,\oo}}$, the set
of all subsets of $\Oo_{i,\oo}$.)
The interpretation is that $\mu_{i,\oo}(E)$ is the probability that
$i$ assigns to event $E \in \F_{i,\oo}$ in state $\oo$.

\commentout{
We thus model ambiguity or vagueness in a fundamentally different way
than much of the literature. Blume and Board \citeyear{BlumeBoard_2009}
take a message $m$ to be a number, and its interpretation is a
(nondegenerate) random variable with mean $m$; Lipman
\citeyear{Lipman_2009} defines vagueness in terms of mixing over
messages; Weaver \citeyear{Weaver_1949} takes an
information-theoretic perspective. The advantage of our syntactic
approach is that it directly models ambiguity, which makes it
straightforward to model the different ways players can reason about
others' interpretations.
}

Throughout this paper, we make the following assumptions regarding the
probability assignments ${\cal P}_i$, $i \in N$:
\begin{description}
  \item[\textbf{A1}.] For all $\oo \in \Oo$, $\Oo_{i,\oo} = \Pi_i(\oo)$.
  \item[\textbf{A2}.] For all $\oo \in \Oo$, if $\oo' \in \Pi_i(\oo)$,
 then ${\cal P}_i(\oo') = {\cal P}_i(\oo)$.
  \item[\textbf{A3}.] For all $j \in N, \oo, \oo' \in \Oo$,
      $\Pi_i(\oo) \cap \Pi_j(\oo') \in \F_{i,\oo}$.
\end{description}
Furthermore, we make the following joint assumption on players'
interpretations and information partitions:
\begin{description}
  \item[\textbf{A4}.] For all $\oo \in \Oo$, $i \in N$, and
            primitive proposition $p \in \Phi$, $\Pi_i(\oo)
      \cap [[p]]_i \in \F_{i,\oo}$.
\end{description}
These are all standard assumptions.  A1 says that the set of states to
which player $i$ assigns probability at state $\oo$ is just the set
$\Pi_i(\oo)$ of worlds that $i$ considers possible at state $\oo$.
A2 says that the probability space used is the same at all the worlds in
a cell of player $i$'s partition.  Intuitively, this says that
player $i$ knows his probability space. Informally, A3 says that player
$i$ can assign a probability to each of $j$'s cells, given his information.
A4 says that primitive propositions (as interpreted by player $i$) are
measurable
according to player $i$.
\commentout{
Finally, say that a structure is
\emph{countably partitioned} if for each player
$i$, the information partition $\Pi_i$ has countably many elements,
i.e., $\Pi_i$ is a finite or countably infinite collection of subsets of
$\Omega$. In the next section, we show that in countably partitioned
structures that satisfy
the assumptions above, a player's beliefs can be thought of as being generated from a prior.
}

\subsection{Prior-generated beliefs and the common-prior assumption}

One assumption that we do not necessarily make, but want to examine in
this framework, is the common-prior assumption.
The common-prior assumption is an instance of a more general assumption,
that beliefs are generated from a prior, which we now define.  The
intuition is that players start with a prior probability; they then
update the prior in light of their information.  Player $i$'s information
is captured by her partition  $\Pi_i$. Thus, if $i$'s prior is $\nu_i$,
then we would expect $\mu_{i,\oo}$ to be $\nu_i(\cdot \mid \Pi_i(\oo))$.

\begin{definition}
An epistemic probability structure $M = (\Oo, (\Pi_j)_{j \in
N}, ({\cal P}_j)_{j \in N}, (\pi_j)_{j \in N})$ has
\emph{prior-generated beliefs}
(\emph{generated by $(\F_1,\nu_1),\ldots, (\F_n,\nu_n)$})
if, for each player $i$, there exist
probability spaces $(\Oo,\F_i,\nu_i)$ such that
\begin{itemize}
\item for all $i,j \in N$ and $\oo \in \Oo$, $\Pi_j(\oo) \in \F_i$;
\item for all $i \in N$ and $\oo \in \Oo$, ${\cal P}_i(\oo) =
    (\Pi_i(\oo), \F_i \mid \Pi_i(\oo), \mu_{i,\oo})$,
    where $\F_i\mid\Pi_i(\oo)$ is the restriction of $\F_i$ to $\Pi_i(\oo)$,%
\footnote{Recall that the restriction of ${\cal F}_i$ to
$\Pi_i(\oo)$ is the $\sigma$-algebra $\{B \cap \Pi_i(\oo): B \in {\cal
F}_i\}$.} %
and $\mu_{i,\oo}(E) =
    \nu_i(E \mid \Pi_i(\oo))$ for all $E \in
    \F_i \mid \Pi_i(\oo)$ if $\nu_i(\Pi_i(\oo)) > 0$.
(There are no constraints on $\nu_{i,\oo}$ if $\nu_i(\Pi_i(\oo)) = 0$.)
\end{itemize}
\end{definition}
It is easy to check that if $M$ has prior-generated beliefs, then
$M$ satisfies A1, A2, and A3.
More interestingly for our purposes, the converse also holds for a large
class of structures.
Say that a structure is
\emph{countably partitioned} if for each player
$i$, the information partition $\Pi_i$ has countably many elements,
i.e., $\Pi_i$ is a finite or countably infinite collection of subsets of
$\Omega$.

\begin{proposition}\label{prop:priorgenerated}
If a structure $M$ has
prior-generated beliefs, then
$M$ satisfies A1, A2, and A3. Moreover, every countably partitioned structure that
satisfies A1, A2, and A3 is one with prior-generated beliefs,
with the priors $\nu_i$ satisfying $\nu_i(\Pi_i(\oo))>0$ for each player
$i \in N$ and state $\oo \in \Oo$.%
\end{proposition}
\begin{proof}
The first part is immediate. To prove the second claim, suppose that $M$
is a structure satisfying A1--A3. Let $\F_i$ be the
unique algebra generated by $\cup_{\oo \in \Oo} \F_{i,\oo}$. To define
$\nu_i$, if there are $N_i < \infty$ cells in the partition $\Pi_i$,
define $\nu_i(\oo) = \tfrac{1}{N_i} \mu_{i,\oo}(\oo)$. Otherwise, if the
collection $\Pi_i$ is countably infinite, order the elements of $\Pi_i$
as $p_i^1, p_i^2, \ldots$.  Choose some state $\oo_k \in p_i^k$ for each
$k$, with associated probability space ${\cal P}_i(\oo_k) =
(\Oo_{i,\oo_k}, \F_{i,\oo_k}, \mu_{i,\oo_k})$.
By A2, each choice of $\oo_k$ in $p_i^k$ gives the
same probability measure $\mu_{i,\oo_k}$.  Define $\nu_i = \sum_k
\tfrac{1}{2^k} \mu_{i,\oo_k}$.
It is easy to see that
$\nu_i$ is a probability measure on $\Oo$, and that
$M$ is generated by $(\F_1,\nu_1),\ldots, (\F_n,\nu_n)$.
\end{proof}

Note that the requirement that that $M$ is countably partitioned is necessary to
ensure that we can have
$\nu_i(\Pi_i(\oo))>0$ for each player $i$ and state $\oo$.

In light of Proposition~\ref{prop:priorgenerated}, when it is
convenient, we will talk of a structure satisfying A1--A3 as being
\emph{generated by} $(\F_1,\nu_1), \ldots, (\F_n,\nu_n)$.

The common-prior assumption is essentially just the special case of
prior-generated beliefs where all the
priors are identical.  We make one additional technical assumption. To state this
assumption, we need one more definition. A state $\oo' \in
\Oo$ is \emph{$G$-reachable} from $\oo \in \Oo$, for $G \subseteq N$, if
there exists a
sequence $\oo_0, \ldots, \oo_m$ in $\Oo$ with $\oo_0 = \oo$ and
$\oo_m = \oo'$, and $i_1, \ldots, i_m \in G$ such that
$\oo_\ell \in \Pi_{i_\ell}(\oo_{\ell -1})$. Denote by $R_G(\oo)
\subseteq \Oo$ the set of states $G$-reachable from $\oo$.

\begin{definition} \label{def:CPA}
An epistemic probability structure $M = (\Oo, (\Pi_j)_{j \in
N}, ({\cal P}_j)_{j \in N}, (\pi_j)_{j \in N})$ satisfies the
\emph{common-prior assumption} (\emph{CPA}) if there exists a probability space
$(\Oo,\F,\nu)$ such that $M$ has prior-generated beliefs generated by
$((\F,\nu), \ldots, (\F,\nu))$, and $\nu(R_N(\oo)) > 0$ for all $\oo \in
\Oo$.
\end{definition}
As shown by Halpern~\citeyear{Halpern_2002}, the assumption that $\nu(R_N(\oo))
> 0$ for each $\oo \in \Oo$ is needed for Aumann's
\citeyear{Aumann_1976} impossibility result.

\subsection{Capturing ambiguity}

We use epistemic probability structures to give meaning to formulas.
Since primitive propositions are interpreted relative to players,
we must allow the interpretation of arbitrary formulas to depend on the
player as well. Exactly how we do this depends on what further assumptions we
make about what
players know about each other's interpretations.  There are many
assumptions that could be made. We focus on two of them here, ones that
we believe arise in applications of interest,
and then
reconsider them under the assumption that there may be some ambiguity
about the partitions.
\paragraph{Believing there is no ambiguity}

The first approach is appropriate for situations where players
may interpret statements differently, but it does not occur to them
that there is another way of interpreting the statement.
Thus, in this
model, if there is a public announcement, all players will think that
their interpretation of the announcement is common knowledge.
We write $(M,\oo,i) \vDasho \ffi$ to denote that $\ffi$ is true at
state $\oo$ according to player $i$
(that is, according to $i$'s interpretation of the primitive
propositions in $\ffi$).
The superscript $\ou$ denotes
\emph{outermost scope}, since the formulas are interpreted relative to
the ``outermost'' player, namely the player $i$ on the left-hand side of
$\vDasho$.  We define $\vDasho$, as usual, by induction.

If $p$ is a primitive proposition,
\[
(M,\oo,i) \vDasho p  \  \text{iff} \   \pi_i(\oo)(p) = \mathbf{true}.
\]
This just says that player $i$ interprets a primitive proposition $p$
according to his interpretation function $\pi_i$.  This clause is common
to all our approaches for dealing with ambiguity.

For conjunction and negation, as is standard,
\begin{gather*}
(M,\oo,i) \vDasho \neg \ffi  \  \text{iff}  \
(M,\oo,i) \not{\vDash}^{\ou} \ffi,\\
(M,\oo,i) \vDasho  \ffi \wedge \psi \   \text{iff} \
  (M,\oo,i) \vDasho \ffi \text{ and } (M,\oo,i) \vDasho \psi.
\end{gather*}

Now consider a probability formula of the form $a_1 \pr_j(\ffi_1) +
\ldots + a_k \pr_j(\ffi_k) \geq b$.
The key feature that distinguishes this semantics is how $i$ interprets
$j$'s beliefs.
This is where we capture the intuition that it does not
occur to $i$ that there is another way of interpreting the formulas other
than the way she does.  Let
\[[[\ffi]]^{\ou}_i = \{\oo: (M,\oo,i) \vDasho \ffi\}.\]
Thus, $[[\ffi]]^{\ou}_i$ is the event consisting of the set of states
where $\ffi$ is true, according to $i$.
Note that A1 and A3 guarantee that the restriction of $\Oo_{j,\oo}$ to $\Pi_i(\oo)$ belongs to $\F_{i,\oo}$.
Assume inductively that $[[\ffi_1]]^{\ou}_i \cap
\Omega_{j,\oo}, \ldots, [[\ffi_k]]^{\ou}_i \cap  \Omega_{j,\oo}
\in \F_{j,\oo}$.  The base case of this induction, where
$\ffi$ is a primitive proposition, is immediate from A3 and A4,
and the
induction assumption clearly extends to negations and conjunctions. We now define
\begin{multline*}\label{eq:prob_formula}
(M,\oo,i) \vDasho a_1 \pr_j(\ffi_1) + \ldots + a_k \pr_j(\ffi_k)
\geq b   \ \text{iff} \ \nonumber\\ a_1
\mu_{j,\oo}([[\ffi_1]]^{\ou}_i \cap \Omega_{j,\oo}) + \ldots + a_k
\mu_{j,\oo}([[\ffi_k]]^{\ou}_i \cap \Omega_{j,\oo}) \geq b.
\end{multline*}
Note that it easily follows from A2 that
$(M,\oo,i) \vDasho a_1 \pr_j(\ffi_1) + \ldots + a_k \pr_j(\ffi_k)
\geq b$ if and only if
$(M,\oo',i) \vDasho a_1 \pr_j(\ffi_1) + \ldots + a_k \pr_j(\ffi_k)
\geq b$ for all $\oo' \in \Pi_j(\oo)$.  Thus,
$[[a_1 \pr_j(\ffi_1) + \ldots + a_k \pr_j(\ffi_k) \geq b]]_i$ is a union of
cells of $\Pi_j$, and hence
$[[a_1 \pr_j(\ffi_1) + \ldots + a_k \pr_j(\ffi_k) \geq b]]_i \cap
\Omega_{j,\oo} \in \F_{j,\oo}$.

With this semantics,
according to player $i$, player $j$ assigns $\ffi$ probability
$b$ if and
only if the set of worlds where $\ffi$ holds according to $i$ has
probability $b$ according to $j$.
Intuitively, although $i$ ``understands'' $j$'s probability space, player
$i$ is not aware that $j$ may interpret $\ffi$ differently from the way
she ($i$) does. That $i$ understands $j$'s probability space
is plausible if we assume that there is a common prior and that $i$
knows $j$'s partition (this knowledge is embodied in the assumption that
$i$ intersects $[[\ffi_k]]^{\ou}_i$ with $\Omega_{j,\oo}$ when assessing
what probability $j$ assigns to $\ffi_k$).\footnote{Note that at state $\oo$, player $i$ will not in general know that it is state
$\oo$.  In particular, even if we assume that $i$ knows which element of
$j$'s partition contains $\oo$, $i$ will not in general know which of
$j$'s cells describes $j$'s current information.  But we assume that
$i$ does know that \emph{if} the state is $\oo$, then $j$ information is
described by $\Oo_{j,\oo}$.  Thus, as usual, ``$(M,i,\oo) \vDasho
\ffi$'' should
perhaps be understood as ``according to $i$, $\ffi$ is true if the
actual world is $\oo$''. This interpretational issue arises even
without
ambiguity in the picture.}

Given our interpretation of probability formulas, the interpretation of
$B_j \ffi$ and $EB^k \ffi$ follows.
For example,
\[
(M,\oo,i) \vDasho B_j \ffi \ \text{iff} \ \mu_{j,\oo}([[\ffi]]^{\ou}_i) = 1.
\]
For readers more used to belief defined in terms of a possibility
relation, note that if the probability measure $\mu_{j,\oo}$ is
\emph{discrete} (i.e., all sets are $\mu_{j,\oo}$-measurable, and
$\mu_{j,\oo}(E) = \sum_{\oo' \in E} \mu_{j,\oo}(\oo')$ for all
subsets $E \subset \Pi_j(\oo)$),
we can define
$\B_j = \{(\oo,\oo') : \mu_{j,\oo}(\oo') > 0\}$; that is, $(\oo,\oo')
\in \B_j$ if, in state $\oo$, agent $j$ gives state $\oo'$ positive
probability.  In that case,
$(M,\oo,i) \vDasho B_j \ffi$ iff $(M,\oo',i) \vDasho \ffi$
for all $\oo'$ such that $(\oo,\oo') \in \B_j$.
That is, $(M,\oo,i) \vDasho B_j \ffi$ iff $\ffi$ is true according to
$i$ in all the worlds to which $j$ assigns positive probability at $\oo$.

It is important to note that $(M,\oo,i) \vDash \ffi$ does not imply
$(M,\oo,i) \vDash B_i \ffi$: while $(M,\oo,i) \vDasho \ffi$ means
``$\ffi$ is true at $\oo$ according to $i$'s interpretation,'' this does
not mean that $i$ believes $\ffi$ at state $\oo$. The reason is that $i$
can be uncertain as to which state is the actual state. For $i$ to
believe $\ffi$ at $\oo$, $\ffi$ would have to be true (according to
$i$'s interpretation) at all states
to which $i$ assigns positive probability.

Finally, we define
\[
(M,\oo,i) \vDasho CB_G \ffi  \  \text{iff}  \
(M,\oo,i) \vDasho EB^k_G \ffi  \  \text{for $k = 1,2, \ldots$}
\]
for any nonempty subset $G \subseteq N$ of players.

\paragraph{Awareness of possible ambiguity}

We now consider the second way of interpreting formulas.  This is
appropriate for players who realize that other players may interpret
formulas differently.  We write
$(M,\oo,i) \vDashi \ffi$ to denote that $\ffi$ is true at state $\oo$
according to player $i$ using this interpretation, which is called
\emph{innermost scope}. The definition of $\vDashi$ is identical to
that of $\vDasho$ except for the interpretation of probability formulas.
In this case, we have
\begin{multline*}
(M,\oo,i) \vDashi a_1 \pr_j(\ffi_1) + \ldots + a_k \pr_j(\ffi_k)
\geq b  \  \text{iff} \ \nonumber \\ a_1
\mu_{j,\oo}([[\ffi_1]]^\inner_j \cap \Omega_{j,\oo}) + \ldots + a_k
\mu_{j,\oo}([[\ffi_k]]^\inner_j \cap \Omega_{j,\oo}) \geq b,
\end{multline*}
where $[[\ffi]]^\inner_j$ is the set of states $\oo$
such that $(M,\oo,j) \vDashi \ffi$.
Hence, according to player $i$, player $j$ assigns $\ffi$ probability
$b$ if and
only if the set of worlds where $\ffi$ holds according to $j$ has
probability $b$ according to $j$.  Intuitively, now $i$ realizes that
$j$ may interpret $\ffi$ differently from the way that she ($i$) does,
and thus assumes that $j$ uses his ($j$'s) interpretation to evaluate the
probability of $\ffi$.
Again, in the case that $\mu_{j,\oo}$ is discrete, this means that
$(M,\oo,i) \vDashi B_j\ffi$ iff $(M,\oo',j) \vDashi \ffi$ for all $\oo'$
such that $(\oo,\oo') \in \B_j$.

Note for future reference that if $\ffi$ is a probability formula or a
formula of the form
$CB_G \ffi'$, then it is easy to see that
$(M,\oo,i) \vDashi \ffi$ if and only if $(M,\oo,j)\vDashi \ffi$; we
sometimes write $(M,\oo) \vDashi \ffi$ in this case.
Clearly, $\vDasho$ and $\vDashi$ agree in the common-interpretation
case, and we can write $\vDash$.

\paragraph{Ambiguity about information partitions}

Up to now, we have assumed that players ``understand'' each other's
probability spaces.  This may not be so reasonable in the presence of
ambiguity and prior-generated beliefs.  We want to model the following
type of situation. Players receive information, or signals, about the
true state
of the world, in the form of strings (formulas). Each player understands
what signals he and other players receive in different states of the
world, but players may interpret signals differently. For instance, player
$i$ may understand that $j$ sees a red car if $\oo$ is the true state of
the world, but $i$ may or may not be aware that $j$ has a different
interpretation of ``red'' than $i$ does. In the latter case, $i$ does not have a full
understanding of $j$'s information structure.

We would like to think of a player's information as being characterized by a formula (intuitively,
the formula that describes the signals received).  Even if the formulas
that describe each information set are commonly known, in the presence
of ambiguity, they might be interpreted differently.

To make this precise, let $\Phi^*$ be the set of formulas that is
obtained from $\Phi$ by closing off under negation and conjunction.
That is, $\Phi^*$ consists of all propositional formulas that can be
formed from the primitive propositions in $\Phi$.
Since the formulas in $\Phi^*$ are not composed of probability
formulas, and thus do not involve any reasoning about interpretations,
we can extend the function $\pi_i(\cdot)$ to $\Phi^*$ in a
straightforward way, and write $[[\ffi]]_i$ for the set of the states of
the world where the formula $\ffi \in \Phi^*$ is true according to $i$.

The key new assumption we make to model players' imperfect understanding
of the other players' probability spaces is that $i$'s partition cell at
$\oo$
is described by a formula $\ffi_{i,\oo} \in \Phi^*$.  But, of course,
this formula
may be interpreted differently by each player.  We want $\Pi_i(\omega)$
to coincide with $i$'s interpretation of the formula $\ffi_{i,\oo}$.
If player $j$ understands that $i$ may be using a different
interpretation than he does (i.e., the appropriate semantics are the
innermost-scope semantics), then $j$ correctly infers that the set of
states that $i$ thinks are possible in $\oo$ is $\Pi_i(\oo) =
[[\ffi_{i,\oo}]]_i$. But if $j$ does not understand that $i$ may
interpret formulas in a different way (i.e., under outermost scope),
then he thinks that the set of states that $i$ thinks are possible in
$\oo$ is given by $[[\ffi_{i,\oo}]]_j$, and, of course,
$[[\ffi_{i,\oo}]]_j$ may not coincide with $\Pi_i(\oo)$. In any case, we
require that $j$ understand that these formulas form a partition and
that $\omega$ belongs to $[[\ffi_{i,\oo}]]_j$.
Thus, we consider structures that
satisfy A5 and A6 (for outermost scope) or A5 and A6' (for innermost scope), in addition to A1--A4.
\begin{description}
  \item[\textbf{A5.}] For each $i \in N$ and $\oo \in \Omega$, there is
  a formula $\ffi_{i,\oo} \in \Phi^*$ such that $\Pi_i(\oo) =
  [[\ffi_{i,\oo}]]_i$.
  \item[\textbf{A6.}] For each $i, j \in N$, the collection
  $\{[[\ffi_{i,\oo}]]_j: \oo \in \Oo\}$ is a partition of $\Omega$
and for all $\oo \in \Oo$, $\oo \in [[\ffi_{i,\oo}]]_j$.
\item[\textbf{A6$'$.}] For each $i \in N$, the collection
  $\{[[\ffi_{i,\oo}]]_i: \oo \in \Oo\}$ is a partition of $\Omega$
and for all $\oo \in \Oo$, $\oo \in [[\ffi_{i,\oo}]]_i$.
\end{description}
\commentout{
A6$'$ is clearly a weakening of A6. While A6 requires the signals for
player $i$ to induce an information partition according to every player
$j$, the weaker version A6$'$ requires this to hold only for player $i$
himself. Assumption A6$'$ is appropriate when players do not understand the
ambiguity that they face, whereas A6 is the appropriate condition if players
understand each other's interpretation of the signals.
}
Assumption A6 is appropriate for outermost scope: it presumes that player $j$
uses his own interpretation of $\ffi_{i,\oo}$ in deducing the beliefs
for $i$ in $\oo$. Assumption A6$'$ is appropriate for innermost
scope. Note that A6$'$
is a weakening of A6. While A6 requires the signals for
player $i$ to induce an information partition according to every player
$j$, the weaker version A6$'$ requires this to hold only for player $i$
himself.

We can now define analogues of outermost scope and innermost scope
in the presence of ambiguous information. Thus, we define two more truth
relations, $\vDashos$ and $\vDashis$.
(The ``ai'' here stands for ``ambiguity of information''.)
The only difference between $\vDashos$ and $\vDasho$ is in the semantics of probability formulas. In giving the
semantics in a structure $M$, we assume that $M$ has prior-generated
beliefs, generated by $(\F_1,\nu_1),\ldots, (\F_n,\nu_n)$.
As we observed in Proposition \ref{prop:priorgenerated}, this
assumption is without loss of generality
as long as the structure is countably partitioned. However, the choice of prior beliefs \emph{is} relevant, as we shall
see, so we have to be explicit about them.
When $i$ evaluates $j$'s probability at a state
$\oo$, instead of using $\nu_{j,\oo}$, player $i$ uses $\nu_j(\cdot \mid
[[\ffi_{j,\oo}]]_i)$.  When $i=j$, these two approaches agree,
but in general they do not. Thus, assuming that $M$ satisfies A5 and A6
(which are the appropriate assumptions for the outermost-scope
semantics),
we have
\[
\begin{array}{ll}
(M,\oo,i) \vDashos a_1 \pr_j(\ffi_1) + \ldots + a_k \pr_j(\ffi_k)
\geq b \   \text{iff} \\ a_1
\nu_j([[\ffi_1]]^\ous_i \mid [[\ffi_{j,\oo}]]^\ous_i) + \ldots\\
\ \ \  + a_k \nu_j([[\ffi_k]]^\ous_i \mid [[\ffi_{j,\oo}]]^\ous_i) \geq
b,
\end{array}
\]
where $[[\psi]]^\ous_i = \{\oo': (M,\oo,i) \vDashos \psi\}$.

That is, at $\oo \in \Oo$, player $j$ receives the
information (a string) $\ffi_{j,\oo}$, which he interprets as
$[[\ffi_{j,\oo}]]_j$. Player $i$ understands that $j$ receives the
information $\ffi_{j,\oo}$ in state $\oo$, but interprets this as
$[[\ffi_{j,\oo}]]_i$. This models a situation such as the following. In
state $\oo$, player $j$ sees a red car, and thinks possible all states of
the world where he sees a car that is red (according to $j$).
Player $i$ knows that at world $\oo$ player $j$ will see a red car
(although she may not know that the actual world is $\oo$, and thus does
not know what color of car player $j$ actually sees).  However, $i$ has a
somewhat different interpretation of ``red car''
(or, more precisely, of $j$ seeing a red car)
than $j$; $i$'s
interpretation corresponds to the event $[[\ffi_{j,\oo}]]_i$.  Since $i$
understands that $j$'s beliefs are determined by conditioning her prior
$\nu_j$ on her information, $i$ can compute what she believes $j$'s
beliefs are.

We can define $\vDashis$ in an analogous way. Thus, the semantics
for formulas that do not involve probability formulas are as given by
$\vDashi$, while the semantics of probability formulas is defined as
follows (where $M$ is assumed to satisfy A5 and A6$'$, which are the
appropriate assumptions for the innermost-scope semantics):
\[
\begin{array}{ll}
(M,\oo,i) \vDashis a_1 \pr_j(\ffi_1) + \ldots + a_k \pr_j(\ffi_k)
\geq b  \  \text{iff} \\
a_1
\nu_j([[\ffi_1]]^\inners_j \mid [[\ffi_{j,\oo}]]^\inners_j) + \ldots \\
\ \ \ + a_k \nu_j([[\ffi_k]]^\inners_j \mid [[\ffi_{j,\oo}]]^\inners_j) \geq b.
\end{array}
\]
Note that although we have written $[[\ffi_{j,\oo}]]^\inners_i$, since
$\ffi_{j,\oo}$ is a propositional formula, $[[\ffi_{j,\oo}]]^\inners_i =
[[\ffi_{j,\oo}]]^\ous_i = [[\ffi_{j,\oo}]]^\ou_i =
[[\ffi_{j,\oo}]]^\inner_i$.
It is important that $\ffi_{j,\oo}$ is a propositional formula here;
otherwise, we would have circularities in the definition, and would
somehow need to define $[[\ffi_{j,\oo}]]^\inners_i$.

Again, here it may be instructive to consider the definition of
$B_j \ffi$ in the case that $\mu_{j,\oo}$ is discrete for all $\oo$.
In this case, $\B_j$ becomes the set $\{(\oo,\oo'): \nu_j( \oo' \mid
[[\ffi_{j,\oo}]]^\inners_j) > 0$.
That is, state $\oo'$ is considered possible by player $j$ in state
$\oo$ if agent $j$ gives $\oo'$ positive probability after
conditioning his prior $\nu_j$ on (his
interpretation of) the information $\ffi_{j,\oo}$ he receives in state $\oo$.
With this definition of $\B_j$, we have, as expected,
$(M,\oo,i) \vDashis B_j \ffi$ iff $(M,\oo',i) \vDashis \ffi$ for all
$\oo'$ such that $(\oo,\oo') \in \B_j$.

The differences in the different semantics
arise only when we consider
probability formulas. If we go back to our example with the red car, we now have a situation
where player $j$ sees a red car in state $\oo$, and thinks possible all
states where he sees a red car. Player $i$ knows that in state $\oo$,
player $j$ sees a car that he ($j$) interprets to be red, and that this
determines his posterior. Since $i$ understands $j$'s notion of seeing a
red car, she has a correct perception of $j$'s posterior in each 
state of the world. Thus, the semantics for $\vDashis$ are identical to
those for $\vDashi$ 
(restricted to the class of structures with prior-generated beliefs that
satisfy A5 and A6$'$), though the information partitions are not
predefined, but rather generated by the signals.

Note that, given an epistemic structure $M$ satisfying A1--A4, there are many
choices for $\nu_i$ that allow $M$ to be viewed as being generated by
prior beliefs.  All that is required of $\nu_j$ is that
for all $\oo \in \Oo$ and $E \in \F_{j,\oo}$ such that $E \subseteq [[\ffi_{j,\oo}]]^\ous_j$, it holds that
$\nu_j(E \cap [[\ffi_{j,\oo}]]^\ous_j )/\nu_j([[\ffi_{j,\oo}]]^\ous_j) = \mu_{j,\oo}(E)$.
\commentout{
It easily follows that if
$E, E' \subseteq [[\ffi_{j,\oo}]]^\ous_j = \Pi_j(\oo)$ and $E, E' \in\F_{j,\oo}$,
then if $\nu_j$ and $\nu'_j$ generate the same posterior beliefs for
$j$ and $\nu_j(E') \ne 0$, then $\nu_j'(E') \ne 0$, and
$\nu_j(E)/\nu_j(E') = \nu_j'(E)/\nu_j'(E')$.
}%
However, because $[[\ffi_{j,\oo}]]^\ous_i$ may not be a subset of
$[[\ffi_{j,\oo}]]^\ous_j = \Pi_j(\oo)$,
we can have two prior probabilities $\nu_j$ and $\nu_j'$ that
generate the same posterior beliefs for $j$, and still have
$\nu_j([[\ffi_k]]^\ous_i \mid [[\ffi_{j,\oo}]]^\ous_i) \ne \nu_j'([[\ffi_k]]^\ous_i \mid [[\ffi_{j,\oo}]]^\ous_i)$ for some formulas $\ffi_k$.  Thus, we must
be explicit about our choice of priors here.

\section{The common-prior assumption revisited}\label{sec:examples}

This section applies the framework
developed in the previous sections to understand the implications of
assuming a common prior when there is ambiguity.
The application in Section \ref{sec:exam_AtD} makes use of the outermost-
and innermost-scope semantics, while Section \ref{sec:example_CPA}
considers a setting with ambiguity about information partitions.

\subsection{Agreeing to Disagree} \label{sec:exam_AtD}

The first application we consider concerns the result of Aumann \citeyear{Aumann_1976} that players cannot ``agree to disagree'' if they have a common prior.
As we show now, this is no longer true if players can have different interpretations. But exactly what ``agreeing to
disagree'' means depends on which semantics we use.

\begin{example}{\bf [Agreeing to Disagree]}\label{exam:AtD}
Consider a structure $M$ with a single state $\oo$, such that
$\pi_1(\oo)(p) = \mathbf{true}$ and $\pi_2(\oo)(p) = \mathbf{false}$.
Clearly $M$ satisfies the CPA.  The fact that there is only a single state in
$M$ means that, although the players interpret $p$ differently, there is
perfect understanding of how $p$ is interpreted by each player.
Specifically, taking $G = \{1,2\}$, we have that $(M,\oo)
\vDashi CB_G (B_1 p  \wedge B_2 \neg p)$.
Thus, with innermost scope, according to each player, there is common
belief that they have different beliefs at state $\oo$; that is, they
agree to disagree.

With outermost scope, we do not have an agreement to disagree in the
standard sense, but the players do disagree on what they have common
belief about. Specifically, $(M,\oo,1) \vDasho CB_G p$ and $(M,\oo,2)
\vDasho CB_G \neg p$.
That is, according to player 1, there is common belief of $p$;
and according to player 2, there is common belief of $\neg p$.
To us, it seems that we have modeled a rather common situation here!
\end{example}

Note that in the model of Example~\ref{exam:AtD}, there is maximal
ambiguity: the players disagree with probability 1.  We also have
complete disagreement. As the following
result shows, the less disagreement there is in the interpretation of events,
the closer the players come to not being able to agree to disagree.  Suppose that $M$ satisfies the
CPA, where $\nu$ is the common prior, and that $\ffi \in \Phi^*$ (so
that $\ffi$ is a propositional formula).
Say that $\ffi$ is \emph{only $\epsilon$-ambiguous in $M$} if
the set of states where the players disagree on the interpretation of
$\ffi$ has $\nu$-measure at most $\epsilon$; that is,
$$\nu(\{\oo: \exists i, j ((M,\oo,i) \vDash \ffi \mbox{ and } (M,\oo,j)
\vDash \neg \ffi\})\})  \le \epsilon.$$
We write $\vDash$ here because, as we observed before, all the semantic
approaches agree on propositional formulas, so this definition makes
sense independent of the semantic approach used.  Note that if players
have a common interpretation, then all formulas are 0-ambiguous.

\begin{proposition}
If $M$ satisfies the CPA and $\ffi$ is only $\epsilon$-ambiguous in $M$,
then there cannot exist players $i$ and $j$,
numbers $b$ and $b'$ with $b' > b+ \epsilon$, and a state $\oo$ such
that all states are $G$-reachable from $\oo$ and
$$(M,\oo) \vDashi CB_G((\pr_i(\ffi) < b) \land (\pr_j(\ffi) > b')).$$
\end{proposition}
\begin{proof} Essentially the same arguments as those used by  Aumann
\citeyear{Aumann_1976} can be used to show that if all states are
reachable from $\oo$ and
$(M,\oo) \vDashi CB_G(\pr_i(\ffi) < b)$, then it must be the case that
$\nu([[\ffi]]_i) < b$,
where $\nu$ is the common prior.
Similarly, $\nu([[\ffi]]_j) > b'$.  This
contradicts the assumption that $\ffi$ is only $\epsilon$-ambiguous in
$M$.
\end{proof}

\subsection{Understanding differences in beliefs}\label{sec:example_CPA}

Since our framework separates meaning from message, it is worth asking
what happens if players receive the same message, but interpret it
differently. Aumann \citeyear{Aumann_1987} has argued that  ``people
with different information may legitimately entertain different
probabilities, but there is no rational basis for people who have always
been fed precisely the same information to do so.''
Here we show that
this is no longer true when information is ambiguous, even if players
have a common prior and
fully understand the ambiguity that they face, except under strong
assumptions on players' beliefs about the information that others
receive.
This could happen if players with exactly the same
background and information can interpret things differently, and thus
have different beliefs.

We assume that information partitions are generated by signals, which
may be ambiguous.
That is, in each state of the world $\oo$, each player $i$ receives some
signal $\sigma_{i,\oo}$ that determines the states of the
world he thinks possible; that is, $\Pi_i(\oo) =
[[\rec_i(\sigma_{i,\oo})]]_i$,
where $\rec_i(\sigma_{i,\oo}) \in \Phi^*$ is ``$i$ received
$\sigma_{i,\oo}$.''
As usual, we
restrict attention to structures with prior-generated beliefs that
satisfy A5 and A6$'$ when considering innermost-scope semantics and
A5 and A6 when considering outermost-scope semantics.

In any given state, the signals that determine the states that players
think are possible may be the same or may differ across players. Following Aumann \citeyear{Aumann_1987}, we are
particularly interested in the former case.
Formally, we say that $\sigma_\oo$ is a
\emph{common signal} in $\oo$ if $\sigma_{i,\oo} =
\sigma_\oo$ for all $i \in N$. For example, if players have a common interpretation, and all players observe a red car in state $\oo$, then
$\sigma_\oo$ is ``red car'', while $\rec_i(\sigma_\oo)$ is ``$i$ observes a
car that is red.''  The fact that ``red car'' is a
common signal in $\oo$ means that all players in fact observe a red
car in state $\oo$.  But assuming that players have received a common
signal does not imply that they have the same posteriors, as the next
example shows:
\begin{example}\label{exam:new}
There are two players, 1 and 2, and three states, labeled $\oo_1, \oo_2,
\oo_3$. The common prior gives each state equal probability, and players
have the same interpretation. In $\oo_1$, both players receive signal
$\sigma$; in $\oo_2$, only 1 does; in $\oo_3$, only 2 receives
$\sigma$. The primitive proposition $p$ is true in $\oo_1$ and $\oo_2$,
and the primitive proposition $q$ is true in $\oo_1$ and $\oo_3$. In
state $\oo_1$, both players receive signal $\sigma$, but player 1
assigns probability 1 to $p$ and probability $\tfrac{1}{2}$ to $q$,
while 2 gives probability $\tfrac{1}{2}$ to $p$ and probability 1 to
$q$.
\end{example}
Thus, players who receive a common signal can end up having a different
posterior over formulas, even if they have a common prior and the same
interpretation. The problem is that even though players have received
the same signal, they do not know that the other has received it, and
they do not know that the other knows they have received it, and so
on. That is, the fact that players have received a
common signal in $\oo$ does not imply that the signal is common
knowledge in $\oo$. We say that signal $\sigma$ is a \emph{public signal
at state $\oo$}
if $(M,\oo) \vDash CB_N (\land_{i \in N} \rec_i(\sigma))$: it is
commonly believed at $\oo$ that all players received $\sigma$.

For the remainder of this section, we will be considering structures
with a common prior $\nu$.  To avoid dealing with topological issues, we
assume that $\nu$ is a discrete measure.
Of course, if the common prior $\nu$ is
discrete, then so are all the measures $\mu_{i,\oo}$.
Let $\Supp(\mu)$ denote the support of the probability measure $\mu$.
If $\mu$ is discrete, then $\Supp(\mu) = \{\oo: \mu(\oo) \ne 0\}$.

Even though common signals are not sufficient for players to have the
same beliefs, as Example \ref{exam:new} demonstrates, Aumann's claim
does hold for common-interpretation structures if players receive a
public signal 
(provided that they started with a common prior):
\begin{proposition}\label{obs:Aumann}
If  $M$ is a common-interpretation structure with a common
prior, and $\sigma$ is a public signal at $\oo$, then players' posteriors are
identical at $\oo$: for all $i, j \in N$ and
$E \in {\cal F}$,
\[
\mu_{i,\oo}(E \cap \Pi_i(\oo)) = \mu_{j,\oo}(E \cap \Pi_j(\oo)).
\]
In particular, for any formula $\ffi$,
\[
\mu_{i,\oo}([[\ffi]] \cap \Pi_i(\oo)) = \mu_{j,\oo}([[\ffi]] \cap \Pi_j(\oo)).
\]
\end{proposition}
\begin{proof}
Let $\nu$ be the common prior in $M$.
By assumption, $\Pi_i(\oo) = [[\rec_i(\sigma)]]$ for all players $i \in
N$.  Since $\sigma$ is public, we have that $(M,\oo) \vDash
B_i(\rec_j(\sigma))$.  Thus, $(M,\oo') \vDash \rec_j(\sigma)$ for all
$\oo' \in \Supp(\mu_{i,\oo})$.  It follows that $\Supp(\mu_{i,\oo})
\subseteq \Pi_j(\oo)$ for all players $i$ and $j$.  Since $\mu_{i,\oo}(\Supp(\mu_{i,\oo})) = 1$,  we have that
$\nu(\Pi_j(\oo) \mid \Pi_i(\oo)) = \nu(\Pi_i(\oo) \mid \Pi_j(\oo)) = 1$.
Thus, for all $E \in \F$, we must have $\nu(E \mid \Pi_i(\oo)) = \nu(E
\mid \Pi_j(\oo)$.  The result now follows immediately.
\end{proof}

There is another way of formalizing the assumption that (it is commonly
believed that) players are `` fed the same information''; namely, we say
that if one player $i$ receives a signal $\sigma$ then so do all others.
Formally, a signal $\sigma$ is a \emph{shared signal} at state
$\oo$ if $(M,\oo) \vDash \land_{i,j \in N} CB_N(\rec_i(\sigma) \iff
\rec_j(\sigma))$.
If there is no ambiguity, a signal is shared iff it is
public;
we leave the straightforward proof
(which uses ideas from the proof of Proposition~\ref{obs:Aumann}) to the
reader.
\begin{proposition}\label{prop:shared_common_no_amb}
If $M$ is a common-interpretation structure, and
$\sigma_\oo$ is received at state $\oo$ by all players, then
$\sigma_\oo$ is a public signal at $\oo$ iff $\sigma$ is a shared signal at $\oo$.
\end{proposition}
The assumption that signals are public or shared is quite strong:
one requires common belief that a particular signal is received (and so
precludes any uncertainties about what one player believes that other
players believe that others have received), while the other requires common
belief that different players always receive the same signal (and, similarly,
precludes uncertainties about what is received).

What happens if we introduce ambiguity? If the signal itself is a
propositional formula (which is the case in
many cases of interest), then players may interpret the signal
differently; that is, we may have $[[\sigma_\oo]]_i \ne
[[\sigma_\oo]]_j$ for $i \ne j$. Moreover, players may have a different
interpretation of observing a given signal, i.e., it is possible that
$[[\rec_i(\sigma_\oo)]]_i \ne [[\rec_i(\sigma_\oo)]]_j$. Going back to
our example of the red car, different players may interpret ``red car''
differently, and they may interpret the notion of observing a red car
differently. In addition, it is now possible that players have the same
posteriors over events, but not over formulas, or vice versa, given that
they may interpret formulas differently.

If we assume that players are not aware that there is ambiguity, then we retain the equivalence between shared signals and common signals, and players' posteriors over formulas coincide after receiving a public signal. However, they may have different beliefs over events:
\begin{proposition}\label{prop:public_signals1}
If $M$ is a structure satisfying A5 and A6, and
$\sigma$ is received at state $\oo$ by all players, then
$\sigma$ is a public signal at $\oo$ iff $\sigma$ is a shared signal
at $\oo$ under outermost-scope semantics. Moreover, if
$M$ has a common prior, and $\sigma$ is a public signal at $\oo$, then
players' posteriors on formulas are identical at $\oo$; that is, for all
formulas $\psi$, we have
\[
\mu_{i,\oo}([[\psi]]_i^{\ous} \cap \Pi_i(\oo)) =
\mu_{j,\oo}([[\psi]]_j^\ous \cap \Pi_j(\oo)).
\]
However, players' posteriors on events may differ; that is, there may
exist some $E$ such that
\[
\mu_{i,\oo}(E \cap \Pi_i(\oo)) \ne \mu_{j,\oo}(E \cap \Pi_j(\oo)).
\]
\end{proposition}
We leave the straightforward arguments to the reader.

The situation for innermost scope presents an interesting contrast. A first observation is that public signals and shared signals are no longer equivalent:
\begin{example}\label{xam:critical}
Consider a structure $M$ with two players, where $\Oo = \{\oo_{11},
\oo_{12}, \oo_{21}, \oo_{22}\}$.  Suppose that $[[\rec_1(\sigma)]]_1 =
[[\rec_2(\sigma)]]_1 =\{\oo_{11},\oo_{12}\}$, and
$[[\rec_1(\sigma)]]_2 =
[[\rec_2(\sigma)]]_2 =\{\oo_{11},\oo_{21}\}$.  Assume that the beliefs
in $M$ are generated by a common prior that gives each state probability
$1/4$.  Clearly $(M,\oo_{11}) \vDashis CB_N(\rec_1(\sigma) \iff
\rec_2(\sigma)) \land \neg CB_N(\rec_1(\sigma))$.
\end{example}
The problem in Example~\ref{xam:critical} is that, although the signal
is shared, the players don't interpret receiving the signal the same
way. It is not necessarily the case that player 1 received
$\sigma$ from player 1's point of view iff player 2 received $\sigma$
from player 2's point of view. The assumption that players receive
shared signals is not strong enough to ensure that they have identical
posteriors, either over formulas or over events. In Example
\ref{xam:critical}, for example, players clearly have different
posteriors on the event $\{\oo_{11},\oo_{12}\}$ in state $\oo_{11}$;
similarly, it is not hard to show that players can have different
posteriors over formulas. Say that $\sigma$ is \emph{strongly
shared at state $\oo$} if
\begin{itemize}
\item $(M,\oo) \vDashis \land_{i,j} CB_N( \rec_i(\sigma) \iff
\rec_j(\sigma))$; and
\item $(M,\oo) \vDashis \land_{i,j}  CB_N(B_i(\rec_i(\sigma)) \iff
B_j(\rec_j(\sigma)))$.
\end{itemize}
The second clause says that it is commonly believed at $\oo$ that each
player believes that he has received $\sigma$ iff each of the other
players believes that he has received $\sigma$. This clause is implied
by the first in common-interpretation structures and with outermost
scope, but not with innermost scope.

\begin{proposition}
If  $M$ is a structure satisfying A5 and A6$'$,
and $\sigma$ is received at $\oo$ by all players, then
$\sigma$ is a public signal at $\oo$ iff $\sigma$ is a strongly shared
signal at $\oo$ under the innermost-scope semantics. If  $M$ is a structure with a common prior
and $\sigma$ is a public signal at $\oo$, then
players' posteriors over events are identical at $\oo$:
for all $i, j \in N$ and all
$E \in {\cal F}$,
\[
\mu_{i,\oo}(E \cap \Pi_i(\oo)) = \mu_{j,\oo}(E \cap \Pi_j(\oo)).
\]
However, players' posteriors on formulas may differ; that is, for some
formula $\psi$, we could have that
\[
\mu_{i,\oo}([[\psi]]_i^{\inners} \cap \Pi_i(\oo)) \ne
\mu_{j,\oo}([[\psi]]_j^\inners \cap \Pi_j(\oo)).
\]
\end{proposition}
These results emphasize the effect of ambiguity on shared and public
signals.

\section{Common priors or common interpretations?}\label{sec:equivalence}

As Example~\ref{exam:AtD} shows, we can have agreement to disagree with
ambiguity under the $\vDashi$ semantics (and thus, also the $\vDashis$
semantics).  We also know that we can do this by having heterogeneous
priors. As we now show, structures with ambiguity that
satisfy the CPA have the same expressive power as common-interpretation
structures that do not necessarily satisfy the CPA
(and common-interpretation structures, by definition, have no ambiguity).
On the other hand, common-interpretation structures with heterogeneous
priors are more general than structures with ambiguity and common
priors.

To make this precise, we consider what formulas are valid in structures
with or without ambiguity or a common prior. To define what it means for
a formula to be valid, we need some more notation.
Fix a nonempty, countable set $\Psi$ of primitive propositions, and let
$\M(\Psi)$ be the class of all structures that satisfy A1--A4 and that are
defined over some nonempty subset $\Phi$ of $\Psi$ such that $\Psi
\setminus \Phi$ is countably infinite.\footnote{Most of our results hold if we just consider the set of
structures defined over some fixed set $\Phi$ of primitive propositions.
However, for one of our results, we need to be able to add fresh
primitive propositions to the language.  Thus, we allow the
set $\Phi$ of primitive propositions to vary over the
structures we consider, but require $\Psi \setminus \Phi$ to be countably
infinite so that there are always ``fresh'' primitive propositions that
we can add to the language.}
Given a subset $\Phi$ of $\Psi$, a formula $\ffi \in {\cal L}_n^C(\Phi)$, and a structure $M \in \M(\Psi)$ over $\Phi$, we say that
$\ffi$ is \emph{valid in $M$ according to outermost scope},
and write $M \vDasho \ffi$, if $(M,\oo,i)\vDasho \ffi$ for all
$\oo \in \Oo$ and $i \in N$.
Given $\ffi \in \Psi$, say that  $\ffi$ is \emph{valid according to
outermost scope} in a class $\N \subseteq \M(\Psi)$ of structures, and write $\N \vDasho \ffi$,
if $M \vDasho \ffi$ for all $M \in \N$ defined over a set $\Phi \subset \Psi$ of primitive
propositions that includes all the primitive
propositions that appear in $\ffi$.

We get analogous definitions by replacing $\vDasho$ by
$\vDashi$, $\vDashos$ and $\vDashis$ throughout (in the latter two
cases, we have to restrict $\N$ to structures that satisfy
A5 and A6 or A6$'$, respectively, in addition to A1--A4).
Finally, given a class of structures
$\N$, let $\N_c$ be the subclass of $\N$ in
which players have a common interpretation. Thus, $\M_c(\Psi)$ denotes the structures in $\M(\Psi)$ with a common interpretation.
Let $\Mis(\Psi)$ denote all structures in $\M(\Psi)$ with prior-generated
beliefs that satisfy A5 and A6 (where we assume that the prior $\nu$
that describes the initial beliefs is given explicitly).\footnote{For ease of exposition, we assume A6 even when dealing with
innermost scope. Recall that A6 implies A6$'$, which is actually the
appropriate assumption for innermost scope.}
Finally, let $\Mcpa(\Psi)$ (resp., $\Miscpa(\Psi)$)
consist of the structures in $\M(\Psi)$ (resp., $\Mis(\Psi)$) satisfying
the CPA.

\begin{proposition}\label{pro:equivalence}
For all formulas $\ffi \in {\cal L}_n^C(\Psi)$,
the following are equivalent:
\begin{itemize}
\item[(a)] $\M_c(\Psi) \vDash \ffi$;
\item[(b)] $\M(\Psi) \vDasho \ffi$;
\item[(c)] $\M(\Psi) \vDashi \ffi$;
\item[(d)] $\Mis_c(\Psi) \vDash \ffi$;
\item[(e)] $\Mis(\Psi) \vDashos \ffi$;
\item[(f)] $\Mis(\Psi) \vDashis \ffi$.
\end{itemize}
\end{proposition}
\begin{proof}
Since the set of structures with a common interpretation
is a subset of the set of structures, it is immediate that (c) and (b)
both imply (a). Similarly, (e) and (f) both imply (d).
The fact that (a) implies (b) is also immediate. For
suppose that $\M_c(\Psi) \vDash \ffi$ and that $M = (\Oo, (\Pi_j)_{j \in N},
({\cal P}_j)_{j \in N}, (\pi_j)_{j \in N}) \in  \M(\Psi)$ is a structure
over a set $\Phi \subset \Psi$ of primitive propositions that contains the primitive propositions that appear in $\ffi$. We must show
that $M \vDasho \ffi$.  Thus, we must show
that $(M,\oo,i) \vDasho \ffi$ for all $\oo \in \Oo$ and $i \in N$. Fix
$\oo \in \Oo$ and $i \in N$, and let $M'_i = (\Oo, (\Pi_j)_{j \in N},
({\cal P}_j)_{j \in N}, (\pi'_j)_{j \in N})$, where
$\pi'_j = \pi_i$ for all $j$. Thus, $M'_i$ is a common-interpretation
structure over $\Phi$, where the
interpretation coincides with $i$'s interpretation in
$M$. Clearly $M'_i$ satisfies A1--A4, so $M'_i \in \M_c(\Psi)$. It is easy to
check that $(M,\oo,i) \vDasho \psi$ if and only if
$(M'_i,\oo,i) \vDash \psi$ for all states $\oo \in \Oo$ and all formulas
$\psi \in {\cal L}_n^C(\Phi)$. Since $M'_i \vDash \ffi$, we must have
that $(M,\oo,i) \vDasho \ffi$, as desired.

To see that (a) implies (c), given a structure
$M = (\Oo, (\Pi_j)_{j \in N}, ({\cal P}_j)_{j \in N}, (\pi_j)_{j \in N})
\in  \M(\Psi)$ over some set $\Phi \subset \Psi$ of primitive
propositions and a player $j \in N$, let $\Oo_j$ be a disjoint copy of
$\Oo$; that is, for every
state $\oo \in \Oo$, there is a corresponding state $\oo_j \in \Oo_j$.
Let $\Oo' = \Oo_1 \cup \ldots \cup \Oo_n$.
Given $E \subseteq \Oo$, let the corresponding subset $E_j \subseteq
\Oo_j$ be the set $\{\oo_j: \oo \in E\}$, and let $E'$ be the subset of
$\Oo'$ corresponding to $E$, that is, $E' = \{\oo_j: \oo \in E, j \in
N\}$.

Define $M' = (\Oo', (\Pi'_j)_{j \in N}, ({\cal P}'_j)_{j \in N},
(\pi'_j)_{j \in N})$, where  $\Oo' = \Oo_1 \cup \ldots \cup \Oo_n$ and, for
all $\oo \in \Oo$ and $i,j \in N$, we have
\begin{itemize}
\item $\Pi'_i(\oo_j) = (\Pi_i(\oo))'$;
\item $\pi_i(\oo_j)(p) = \pi_j(\oo)(p)$ for a primitive proposition $p
\in \Phi$;
\item ${\cal P}'_i(\oo_j) = (\Oo'_{i,\oo_j}, \F'_{i,\oo_j}, \mu'_{i,\oo_j})$, where $\Oo'_{i,\oo_j} = \Oo_{i,\oo}'$,
$\F'_{i,\oo_j} = \{E_\ell: E \in \F_{i,\oo}, \ell \in N\}$,
$\mu'_{i,\oo_j}(E_i) = \mu_{i,\oo}(E)$,
$\mu'_{i,\oo_j}(E_\ell) = 0$ if $\ell \ne i$.
\end{itemize}

Thus, $\pi_1 = \cdots = \pi_n$, so that $M'$ is a common-interpretation
structure; on a state $\oo_j$, these interpretations are
all determined by $\pi_j$.  Also note that the support of the probability
measure $\mu'_{i,\oo_j}$ is contained in $\Oo_i$, so for different
players $i$, the probability measures $\mu'_{i,\oo_j}$ have disjoint
supports. Now an easy induction on the structure of formulas shows that$(M',\oo_j) \vDash \psi$ if and only if $(M,\oo,j) \vDashi \psi$ for any formula $\psi \in {\cal L}_n^C(\Phi)$.  It easily
follows that if $M' \vDash \ffi$, then $M \vDashi \ffi$ for all $\ffi \in {\cal L}_n^C(\Phi)$.

The argument that (d) implies (e) is essentially identical to the
argument that (a) implies (b); similarly, the argument that
(d) implies (f) is essentially the same as the argument that (a) implies
(c).  Since $\Mis_c(\Psi) \subseteq \M_c(\Psi)$, (a) implies (d).  To show that (d)
implies (a), suppose that $\Mis_c(\Psi) \vDash \ffi$ for some formula
$\ffi \in {\cal L}_n^C(\Psi)$. Given a structure
$M = (\Oo, (\Pi_j)_{j \in N}, ({\cal P}_j)_{j \in N}, \pi)\in \M_c(\Psi)$ %
over a set $\Phi \subset \Psi$ of primitive propositions that includes
the primitive
propositions that appear in $\ffi$, we want to show that
$(M,\oo,i) \vDash \ffi$ for each state $\oo \in \Oo$ and player $i$.  Fix
$\oo$.  Recall that $R_N(\oo)$ consists of the set of states $N$-reachable
from $\oo$.  Let $M' = (R_N(\oo), (\Pi'_j)_{j \in N}, ({\cal P}'_j)_{j \in N}, \pi')$,
with $\Pi'_j$ and ${\cal P}'_j$ the restriction of $\Pi_j$ and ${\cal
P}_j$, respectively, to the states
in $R_N(\oo)$, be a structure over a set $\Phi'$ of primitive
propositions, where $\Phi'$ contains $\Phi$ and
new primitive propositions that we call
$p_{i,\oo}$ for each player $i$ and state $\oo \in
R_N(\oo)$.\footnote{This is the one argument that needs the assumption that the
set of primitive propositions can be different in different structures in
$\M(\Psi)$, and the fact that every $\Psi \setminus \Phi$ is countable.
We have
assumed for simplicity that the propositions $p_{i,\oo}$ are all in
$\Psi \setminus \Phi$, and that they can be chosen in such a way so that
$\Psi \setminus (\Phi \cup \{p_{i,\oo}: i \in \{1,\ldots,n\}, \oo \in
\Omega\})$  is
countable.}
Note that there are only countably many information sets in $R_N(\oo)$, so
$\Phi'$ is countable. Define $\pi'$ so that it agrees with $\pi$
(restricted to $R_N(\oo)$) on
the propositions in $\Phi$, and so that $[[p_{i,\oo}]]_i = \Pi_i(\oo)$.  Thus,
$M'$ satisfies A5 and A6. It is easy to check that, for all $\oo' \in R_N(\oo)$ and
all formulas $\psi \in {\cal L}_n^C(\Phi)$, we have that $(M,\oo',i)
\vDash \psi$ iff
$(M',\oo',i) \vDash \psi$.  Since $M' \vDash \ffi$, it follows that
$(M,\oo,i) \vDash \ffi$, as desired.
\end{proof}

The proof that (a) implies (c)
shows that, starting from an arbitrary structure $M$, we can
construct a common-interpretation structure $M'$ that is equivalent to
$M$ in the sense that the same formulas hold in both models.
Note that because the probability measures in the structure $M'$ constructed in the proof of Proposition~\ref{pro:equivalence} have disjoint support,
$M'$ does not satisfy the CPA, even if the original
structure $M$ does.
As the next result shows, this is not an accident.
\begin{proposition}\label{pro:equivalence1}
For all formulas $\ffi \in {\cal L}_n^C(\Psi)$, if
either $\Mcpa(\Psi) \vDasho \ffi$, $\Mcpa(\Psi) \vDashi \ffi$,
$\Miscpa(\Psi) \vDashos \ffi$, or $\Miscpa(\Psi) \vDashis \ffi$,
then $\Mcpa_c(\Psi) \vDash \ffi$.
Moreover, if $\Mcpa_c(\Psi) \vDash \ffi$, then $\Mcpa(\Psi) \vDasho \ffi$ and
$\Miscpa(\Psi) \vDashos \ffi$.
However,
in general, if  $\Mcpa_c(\Psi) \vDash \ffi$, then it may not be the case that
$\Mcpa(\Psi) \vDashi \ffi$.
\end{proposition}
\begin{proof} All the implications are straightforward, with proofs along the
same lines as that of Proposition~\ref{pro:equivalence}.
To prove the last claim, let $p \in \Psi$ be a primitive proposition.
Aumann's agreeing to disagree result shows that
$\Mcpa_c(\Psi) \vDash \neg CB_G (B_1 p  \wedge B_2 \neg p)$, while
Example~\ref{exam:AtD} shows that $\Mcpa(\Psi)
\not{\vDash}^{\mathit{in}} \neg CB_G (B_1 p  \wedge B_2 \neg p)$.
\end{proof}

Proposition~\ref{pro:equivalence1} depends on the fact that we are
considering belief and
common belief rather than knowledge and common knowledge,
where knowledge is defined in the usual way, as truth in all
possible worlds:
\[
(M,\oo,i) \vDashi K_i\ffi \ \text{ iff } \ (M,\oo',i) \vDashi \ffi \text{ for all $\oo' \in \Pi_i(\oo)$,}
\]
with $K_i$ the knowledge operator for player $i$, and where we have assumed that $\oo \in \Pi_i(\oo)$ for all $i \in N$ and $\oo \in \Oo$.
Aumann's
result holds if we consider common belief (as long as what we are
agreeing about are judgments of probability and expectation). With
knowledge, there are formulas that are valid with a common
interpretation that are not valid under innermost-scope
semantics when there is ambiguity.For example,
$\M_c(\Psi) \vDash
K_i \ffi \Rightarrow \ffi$, while it is easy to construct a structure
$M$ with ambiguity such that $(M,\oo,1) \vDashi K_2 p \wedge \neg p$.
What is true is that $\M(\Psi) \vDashi (K_1 \ffi \Rightarrow \ffi) \vee
\ldots \vee (K_n \ffi \Rightarrow \ffi)$.
This is because we have $(M,\oo,i) \vDashi K_i \ffi \Rightarrow \ffi$,
so one of $K_1 \ffi \Rightarrow \ffi$, \ldots, $K_n \ffi \Rightarrow
\ffi$ must hold. As shown in \cite{Hal43}, this axiom essentially
characterizes knowledge if there is ambiguity.

As noted above, the proof of Proposition~\ref{pro:equivalence} demonstrates
that, given a structure $M$ with ambiguity and a common prior, we can
construct an equivalent common-interpretation structure $M'$ with
heterogeneous priors, where $M$ and $M'$ are said to be \emph{equivalent
(under innermost scope)} if for every formula $\psi$, $M \vDashi \psi$ if
and only if $M' \vDashi \psi$. The converse does not hold, as the next
example illustrates: when formulas are interpreted using innermost
scope, there is a common-interpretation structure with heterogeneous
priors that cannot be converted into an equivalent structure with
ambiguity that satisfies the CPA.
\begin{example} \label{exam:no_equiv}
We construct a structure $M$ with heterogeneous priors for which there
is no equivalent ambiguous structure that satisfies the CPA. The
structure $M$ has three players, one primitive proposition
$p$, and two states, $\oo_1$ and $\oo_2$. In $\oo_1$, $p$ is
true according to all players; in $\oo_2$, the proposition is
false. Player 1 knows the state: his information partition is $\Pi_1 =
\{\{\oo_1\},\{\oo_2\}\}$. The other players have no information on the
state, that is, $\Pi_i = \{\{\oo_1, \oo_2\}\}$ for $i = 2,3$. Player $2$
assigns probability $\tfrac{2}{3}$ to $\oo_1$, and player 3 assigns
probability $\tfrac{3}{4}$ to $\oo_1$. Hence, $M$ is a
common-interpretation structure with heterogeneous priors. We claim that
there is no equivalent structure $M'$ that satisfies the CPA.

To see this, suppose that $M'$ is an equivalent structure that satisfies
the CPA, with a common prior $\nu$ and a state space $\Oo'$.
As $M' \vDashi \pr_2(p) = \tfrac{2}{3}$ and
$M' \vDashi \pr_3( p) = \tfrac{3}{4}$, we must have
\begin{gather*}
\nu(\{\oo' \in \Oo': (M',\oo',2) \vDashi p\}) = \tfrac{2}{3},\\
\nu(\{\oo' \in \Oo': (M',\oo',3) \vDashi p\}) = \tfrac{3}{4}.
\end{gather*}
Observe that $M \vDashi B_2(p \iff B_1p) \land B_3(p \iff B_1p)$.
Thus, we must have $M' \vDashi B_2(p \iff B_1p) \land B_3(p \iff B_1p)$.
Let $E = \{\oo' \in \Oo': (M',\oo',1) \vDashi B_1 p\}$.  It follows that we
must have $\nu(E) = 2/3$ and $\nu(E) = 3/4$, a contradiction.
\end{example}

Example \ref{exam:no_equiv} demonstrates that there is no structure
$M'$ that is equivalent to the structure $M$ (defined in Example
\ref{exam:no_equiv}) that satisfies the CPA. In fact, as we show now,
an even stronger result holds: In any structure $M'$ that is equivalent
to $M$, whether it satisfies the CPA or not, players
have a common interpretation.
\begin{proposition}\label{prop:equiv_struct_is_common_interpr_struct}
If a structure $M' \in \M(\Psi)$ is equivalent
under innermost scope
to the structure $M$
defined in Example \ref{exam:no_equiv}, then
$M' \in \M_c$.
\end{proposition}
\begin{proof}
Note that $M \vDash p \iff (\pr_1(p)=1)$. Hence, if a structure $M'$ is
equivalent to $M$, we must have that $M' \vDashi p \iff (\pr_1(p)=1)$,
that is, for all $\oo \in \Oo$ and $i \in N$, $(M',\oo,i) \vDashi p \iff
(\pr_1(p)=1)$. By a similar argument, we obtain that for every $\oo \in
\Oo$ and $i \in N$, it must be the case that $(M',\oo,i) \vDashi \neg p \iff
(\pr_1(\neg p)=1)$. Since the truth of a probability formula does not
depend on the player under the innermost-scope semantics, it follows
that for each $i,j \in N$, the interpretations $\pi_i'$ and $\pi_j'$ in
$M'$ coincide. In other words, $M'$ is a common-interpretation
structure.
\end{proof}

\section{Discussion} \label{sec:concl}

We have defined a logic for reasoning about ambiguity, and considered
the tradeoff between having a common prior (so that everyone starts out
with the same belief) and having a common interpretation (so that
everyone interprets all formulas the same way).  We showed that, in a
precise sense, allowing different beliefs is more general than allowing
multiple interpretations.  But we view that as a feature, not a
weakness, of considering ambiguity.  Ambiguity can be viewed as a reason
for differences of beliefs; as such, it provides some structure to these differences.

We have not discussed axiomatizations of the logic.  From
Proposition~\ref{pro:equivalence} it follows that for formulas in
${\cal L}_n^C(\Psi)$, we can get the same axiomatization with respect
to structures in $\M(\Psi)$ for both the $\vDasho$ and $\vDashi$
semantics; moreover, this axiomatization is the same as that for the
common-interpretation case.  An axiomatization for this case is already
given in \cite{FH3}.  Things get more interesting if we consider
$\Mcpa(\Psi)$, the structures that satisfy the CPA.
Halpern  \citeyear{Halpern_2002}
provides an axiom that says that it cannot be common knowledge that
players disagree in expectation, and shows that it can be used to obtain
a sound and complete characterization of common-interpretation
structures with a common prior.  (The axiomatization is actually given
for common knowledge rather than common belief, but a similar result
holds with common belief.) By Proposition~\ref{pro:equivalence1},
the axiomatization remains sound for outermost-scope semantics if
we assume the CPA. However, using Example~\ref{exam:no_equiv}, we can
show that this is no longer the case for the innermost-scope semantics.
The set of formulas valid for innermost-scope semantics in the class of
structures satisfying the CPA is strictly between the set of formulas
valid in all structures and the set of formulas valid for
outermost-scope semantics in the class of structures satisfying the CPA. Finding an elegant complete
axiomatization remains an open problem.

\paragraph{Acknowledgments:} Halpern's work was supported
in part by NSF grants IIS-0534064, IIS-0812045, and
IIS-0911036, by AFOSR grants
FA9550-08-1-0438 and FA9550-09-1-0266, and by ARO grant W911NF-09-1-0281. The work of Kets was supported in part by AFOSR grant FA9550-08-1-0389.

\bibliographystyle{chicagor}
\bibliography{bib_AtD,z,joe}
\end{document}